\documentclass[letterpaper, 10 pt, conference]{ieeeconf}  

\IEEEoverridecommandlockouts                              
\usepackage{amsmath,amssymb,bbm,color,psfrag,amsfonts,mathtools,amsbsy}
\usepackage{dsfont}
\usepackage{graphicx,cuted}                       

\usepackage{wrapfig}
\usepackage{pdfpages}
\usepackage{accents}
\usepackage{lipsum}
\usepackage[T1]{fontenc}
\usepackage[utf8]{inputenc}
\usepackage{mathtools}
\usepackage{multicol}
\usepackage{paralist}   
\usepackage{figlatex}
\let\labelindent\relax
\usepackage{enumitem}
\usepackage{booktabs}  
\usepackage{multirow}
\usepackage[normalem]{ulem}

\usepackage{caption}
\usepackage{subcaption}

\usepackage{cite}
\usepackage[pdftex,pdfpagelabels,bookmarks,hyperindex,hyperfigures]{hyperref}

\newcounter{problem}

\makeatletter
\newcommand\footnoteref[1]{\protected@xdef\@thefnmark{\ref{#1}}\@footnotemark}
\makeatother

\newtheorem{theorem}{Theorem}

\newtheorem{lemma}{Lemma}

\newtheorem{remark}{Remark}
\newtheorem{example}{Example}

\usepackage{color}

\newcommand{\kibitz}[2]{\ifnum\Comments=0\textcolor{#1}{#2}\fi}

\newcommand{\RR}{\mathbb{R}}
\newcommand{\NN}{\mathbb{N}}
\newcommand{\mc}{\mathcal}

\newcommand{\im}{\mathbf{i}}
\newcommand{\CC}{\mathbb{C}}
\newcommand{\ZZ}{\mathbb{Z}}

\title{\LARGE \bf
FNO$^{\angle \theta}$: Extended Fourier neural operator for learning state and optimal control of distributed parameter systems
}

\author{Zhexian Li, Ketan Savla
\thanks{Z. Li is with the Sonny Astani Department of Civil and Environmental Engineering, University of Southern California, Los Angeles, CA 90089, USA. {\tt\small zhexianl@usc.edu}}%
\thanks{K. Savla is with the Sonny Astani Department of Civil and Environmental Engineering, the Daniel Epstein Department of Industrial and Systems Engineering, and the Ming Hsieh Department of Electrical and Computer Engineering, University of Southern California, Los Angeles, CA 90089, USA. {\tt\small ksavla@usc.edu}}
\thanks{
 K. Savla has a financial interest in Xtelligent, Inc.}%
 \thanks{The code is available at \href{https://github.com/zhexianli-usc/Optimal-control/tree/master}{[Code]}}
}

\begin{document}

\maketitle
\thispagestyle{empty}
\pagestyle{empty}

\begin{abstract}
We propose an extended Fourier neural operator (FNO) architecture for learning state and linear quadratic additive optimal control of systems governed by partial differential equations.
Using the Ehrenpreis-Palamodov fundamental principle, we show that any state and optimal control of linear PDEs with constant coefficients can be represented as an integral in the complex domain.
The integrand of this representation involves the same exponential term as in the inverse Fourier transform, where the latter is used to represent the convolution operator in FNO layer.
Motivated by this observation, we modify the FNO layer by extending the frequency variable in the inverse Fourier transform from the real to complex domain to capture the integral representation from the fundamental principle.
We illustrate the performance of FNO in learning state and optimal control for the nonlinear Burgers' equation, showing order of magnitude improvements in training errors and more accurate predictions of non-periodic boundary values over FNO.

\end{abstract}

\section{INTRODUCTION}

Distributed parameter systems, such as systems governed by partial differential equations, are of great importance in various scientific and engineering fields, including fluid dynamics, climate modeling, and control of physical processes.
The domain of such systems typically involves both temporal and spatial dimensions, making prediction and control tasks significantly harder than finite-dimensional systems.
On the one hand, the state space of distributed parameter systems is infinite-dimensional, for which optimal control methods have been developed \cite{lions1970optimal,hinze2008optimization}, but often requiring significant computational resources even for small-scale systems.
On the other hand,
discretizing the spatial domain, also known as spatially lumping the system, leads to finite-dimensional approximation, for which a rich set of tools and methods have been developed for optimal control \cite{bryson2018applied}. 
However, the tradeoff between accuracy and discretization size needs to be carefully investigated to balance the computational efficiency and approximation error for the lumped system.

Recently, there has been a surge of interest in using neural operators to learn solution operators for PDEs, due to the potential in learning fast and accurate surrogate infinite-dimensional models.
Neural operators are mappings between function spaces represented in the form of neural networks \cite{kovachki2023neural}, e.g., from coefficient functions or boundary conditions to PDE solutions.
Every layer in the neural network consists of an integral operator that varies between different types of neural operators.
Among all different neural operators, the Fourier neural operator (FNO) \cite{li2020fourier} uses a convolution operator represented in the Fourier space. 
This simple architecture has been found to provide the smallest approximation error for solutions of various PDEs \cite{takamoto2022pdebench} and to successfully reproduce complicated physical phenomena such as turbulent flows \cite{li2020fourier}.
In addition to solving PDEs, neural operators have also been used to learn optimal control operators \cite{song2026learning} and to assist in the control computation of PDEs such as boundary control with backstepping methods \cite{bhan2023neural}.

The key motivation for using the convolution operator in FNO is the fact that every linear PDE in $\mathbb{R}^n$ with constant coefficients has an elementary solution in convolution form, following the fundamental solution independently established by Ehrenpreis \cite{ehrenpreis1954solution} and Malgrange \cite{malgrange1956existence}.
However, solutions in bounded domains with non-periodic boundary conditions are not in convolution form.
The limitation of the convolution operator becomes apparent when implemented as a Fourier series, which is inherently periodic.
Nevertheless, this limitation can be mitigated in practice by using additional parameters in neural operators \cite{li2020fourier} or Fourier continuation methods \cite{ganeshram2022fc}.

Although the fundamental solution is for systems without boundaries, additional structures for solution representations have been established for general convex (potentially bounded) domains.
Specifically, the Ehrenpreis-Palamodov fundamental principle \cite{ehrenpreis1970fourier,palamodov1970linear} states that any solution to linear PDEs in a convex subset of $\RR^n$ with constant coefficients can be represented as an integral of exponential functions multiplied by polynomials over structured subsets of $\CC^n$.
The proof of this result is nonconstructive and explicit expressions for such an integral representation have been obtained only for one- or two-dimensional space using the Fokas method \cite{fokas1997unified,fokas2008unified,deconinck2014method}.
However, the structural insights provided by the integral representation can be leveraged in the design of neural operators.

Our goal is to design a neural operator that uses the integral representation of the Ehrenpreis-Palamodov fundamental principle in every neural layer to learn state and control operators for systems governed by PDEs.
Despite its abstract nature, we observe that the exponential term in the integral representation from the fundamental principle is in the same form of the inverse Fourier transform, with the difference that the former integral is over the complex domain.
Since the inverse Fourier transform is used in the convolution operator in FNO, it is natural for the newly designed neural operator to inherit the structure of FNO and extend the frequency variable to the complex domain.

The contribution of this paper is the following. First, we show that, in addition to solutions, linear quadratic optimal control for linear PDEs with constant coefficients can also be represented in a similar integral form as the fundamental principle.
Then, we propose an extended FNO architecture, called FNO$^{\angle \theta}$, by introducing a new complex variable in the inverse Fourier transform, whose modulus is the original frequency variable in FNO and phase is a new parameter $\theta$ to learn.
When $\theta=0$, FNO$^{\angle \theta}$ reduces to FNO and inherits the universal approximation property \cite{kovachki2021universal}.
Beyond theoretical analysis for linear PDEs, we present numerical experiments to demonstrate the performance of FNO$^{\angle \theta}$ in learning state and optimal control for the nonlinear Burgers' equation, exhibiting order of magnitude improvements on training errors and more accurate predictions of non-periodic boundary values over FNO.

We conclude this section by introducing key notations used in the paper. We use $\RR, \RR_+, \CC, \ZZ, \NN$ to denote the sets of real numbers, positive real numbers, complex numbers, integers, and natural numbers, respectively.
We use $\text{C}^\infty(\Omega)$ and $\text{L}^2(\Omega)$ to denote the set of infinitely differentiable functions and square-integrable functions defined on $\Omega$, respectively.
The convolution operator is denoted by $\ast$.
The composition of two operators $\mc A$ and $\mc B$ is denoted by $\mc A \circ \mc B$.

\section{PROBLEM FORMULATION}

In this paper, we consider linear quadratic optimal control for systems governed by linear PDEs with constant coefficients.
Let $P$ denote a matrix with entries $P_{j,l}, 1\leq j\leq p, 1\leq l\leq q$.
Each entry $P_{j,l}$ is a polynomial function in $\CC^n$. 
Let $D_j = -\im \partial/\partial_{\xi_j}, j= 1, \ldots, n$, and let $D = (D_1, \ldots, D_n)$ denote the (partial) differential operator. 
Then $P(D)$ is a matrix of linear partial differential operators with constant coefficients.
Following this notation, we consider the system of inhomogeneous PDEs in a convex and open domain $\Omega\subseteq \RR^n$, of the form
\begin{equation}\label{eq:system}
P(D) \phi = u,  
\end{equation}
or, more explicitly,
\begin{equation}\label{eq:system-explicit}
\sum_{l=1}^q P_{j,l}(-\im \partial/\partial_{\xi_1}, \ldots, -\im \partial/\partial_{\xi_n}) \phi_l = u_j, \quad j=1, \ldots, p,
\end{equation} 
where $\phi$ is the state and $u$ is the control input.
In this paper, we focus on smooth functions $u_j$ and $\phi_l$ in $\text{L}^2(\Omega)\cap \text{C}^\infty(\Omega), j=1,\ldots,p$ and $l=1,\ldots,q$.
Note that the system \eqref{eq:system} has solutions in more general spaces, such as the space of distributions \cite{treves2012ehrenpreis}.
\begin{remark}
   In practice, the domain $\Omega$ can denote a spatial domain or a product of a spatial domain and a time interval, i.e., $\Omega = \Omega_x \times \Omega_t$ where $\Omega_x \subseteq \RR^{n-1}$ and $\Omega_t \subseteq \RR_+$.
   For brevity, we use a single variable $\xi$ to denote a point in $\Omega$ and do not distinguish between spatial and temporal components.
\end{remark}
\begin{remark}
   Appropriate boundary conditions are needed to ensure the uniqueness of the solution to the system \eqref{eq:system}.
   We first focus on general solutions to \eqref{eq:system} and postpone the discussion of boundary conditions to numerical experiments.
\end{remark}
\begin{example}
   Throughout the paper, we will use the following one-dimensional heat equation as a running example to illustrate the main results,
   \begin{equation}\label{eq:heat-equation}
      \frac{\partial \phi(x,t)}{\partial t} = \frac{\partial^2 \phi(x,t)}{\partial x^2} + u(x,t).
   \end{equation}
   In this example, $P$ is a single polynomial and $P(\xi) = \im \xi_1 + \xi_2^2$, where $\xi = (\xi_1, \xi_2)$.
\end{example}

The optimal control problem under consideration is to find a control input $u$ that minimizes a cost functional $J(\phi, u)$ of the following form subject to the system \eqref{eq:system},
\begin{equation}\label{eq:objective-functional}
      J(\phi,u) = \int_{\Omega} |\phi(\xi) - \phi_d(\xi)|^2 + \lambda |u(\xi)|^2 d\xi,
\end{equation}
where $\phi_d$ is a given desired state and $\lambda > 0$ is a regularization parameter.
When $\phi_d=0$, this problem is a special case of the classical linear quadratic regulator (LQR) problem for PDEs \cite{curtain2012introduction}.
For $\Omega=\Omega_x \times \Omega_t$, where $\Omega_x = \RR^{n-1}$ and $\Omega_t = \RR_+$, explicit expressions for optimal control have been derived in \cite{bamieh2002distributed}. 
The one-dimensional case where $\Omega_x$ is a finite interval has been studied in \cite{li2024complex}.
However, explicit expressions for optimal control in a general (potentially bounded) domain $\Omega$ are not readily available.
This is relevant for many practical applications where the spatial domain is typically bounded.

In this paper, we propose to develop a neural operator for computing optimal control. 
Although we formulate the optimal control problem for linear PDEs, the neural operator can be applied to nonlinear PDEs as well.
The key to our approach is to leverage an integral representation of state and optimal control in designing neural operator architecture.
The integral representation is derived in the following section.

\section{INTEGRAL REPRESENTATION FOR OPTIMAL CONTROL}
\label{sec:integral-representation}
To derive an integral representation for optimal control, we first present the first-order optimality conditions that are satisfied by any optimal control to \eqref{eq:objective-functional} subject to \eqref{eq:system}.
The optimality conditions consist of a coupled system of PDEs, whose solution can be represented in terms of integrals over the complex domain.
Following the adjoint approach in \cite[Section 1.6.3]{hinze2008optimization}, we obtain the following result.
\begin{lemma}
   Let $u^*$ be an optimal control for \eqref{eq:objective-functional} subject to \eqref{eq:system} and $\phi^*$ be the corresponding optimal state.
   Then, there exists an adjoint state $\psi^*$ such that $u^*$ and $\phi^*$ satisfy the following system of equations,
   \begin{equation}\label{eq:optimality-conditions}
   \begin{cases}
      P(D) \phi^* = u^*, \\
      P(D)^\dagger \psi^* = -(\phi^* - \phi_d), \\
      u^* = \frac{1}{\lambda} \psi^*,
   \end{cases}
   \end{equation}
   where $P(D)^\dagger$ denotes the adjoint of $P(D)$.
\end{lemma}

Substituting the last equation for $u^*$ into the first equation in \eqref{eq:optimality-conditions}, we observe that optimal control $u^*$ and state $\phi^*$ can be represented as a solution to the following system of PDEs,
\begin{equation}\label{eq:optimal-control-pde}
   \begin{bmatrix}
      P(D) & -I \\
      I & \lambda P(D)^\dagger
   \end{bmatrix}
   \begin{bmatrix}
      \phi^* \\
      u^*
   \end{bmatrix}
   = \begin{bmatrix}
      0 \\
      \phi_d
   \end{bmatrix}.
\end{equation}
Note that \eqref{eq:optimal-control-pde} is a system of linear PDEs with constant coefficients and can be written in a similar form to \eqref{eq:system}. 

In the following, we aim to derive an integral representation for optimal control $u^*$ and the corresponding state $\phi^*$ in \eqref{eq:optimal-control-pde}.
We first present the existence of a solution to \eqref{eq:system} in convolution form.
Then, we combine this convolution solution with the Ehrenpreis-Palamodov fundamental principle \cite{ehrenpreis1970fourier,palamodov1970linear} to derive an integral representation for optimal control.

\begin{lemma}[Theorem 2.11 in \cite{stein2011functional}]\label{lemma:fundamental-solution}
   Let $\Omega = \mathbb{R}^n$. Then, there exists a distribution $\tilde{E}$ such that $P(D) \tilde{E} = \delta$, where $\delta$ is the Dirac delta function. Moreover, the convolution of $\tilde{E}$ with the control input $u$, i.e., $\tilde{\phi} = \tilde{E} \ast u$, is a solution to \eqref{eq:system}.
\end{lemma}

\begin{remark}
   FNO is inspired by the existence of a solution $\tilde{\phi}$ in convolution form. 
The key ingredient in FNO is to learn the convolution operator in the Fourier space, which is equivalent to learning a multiplication parameter. 
\end{remark}

The existence of $\tilde{\phi}$ is for the case $\Omega=\RR^n$.
For a general open convex domain $\Omega\subseteq\RR^n$, the Ehrenpreis-Palamodov fundamental principle states that any solution to \eqref{eq:system} with $u=0$ can be represented as an integral of exponential functions multiplied by polynomials over structured subsets in $\CC^n$.

\begin{lemma}[\cite{ehrenpreis1970fourier,palamodov1970linear,harkonen2023gaussian}]\label{lemma:fundamental-principle}
   Let $\bar{\phi}$ denote a solution to the system \eqref{eq:system} with $u=0$. Then, there exists zero sets of polynomials $\bar{\mc V}_j\subset \CC^n, j=1,\ldots,s$, polynomials $\bar{Q}_{j,l}, j=1,\ldots,s, l=1,\ldots,m_j$, and measures $d\bar{\mu}_{j,l}$ whose support lies in $\bar{\mc V}_j$, such that $\bar{\phi}$ can be represented as an integral of the following form,
\begin{equation}\label{eq:ehrenpreis-representation}
   \bar{\phi} = \sum_{j=1}^{s}\sum_{l=1}^{m_j}\int_{\bar{\mc V}_j} e^{\im k \cdot \xi} \bar{Q}_{j,l}(\xi,k) d\bar{\mu}_{j,l}(k).
\end{equation}
Moreover, the function $e^{\im k \cdot \xi} \bar{Q}_{j,l}(\xi,k)$ satisfies the system of PDEs in \eqref{eq:system} for each $k\in \bar{\mc V}_j$.
\end{lemma}

Applying Lemmas~\ref{lemma:fundamental-solution}--\ref{lemma:fundamental-principle} to the coupled state-control system \eqref{eq:optimal-control-pde}, we obtain the following result.

\begin{theorem}\label{thm:integral-representation}
Let $u^*$ denote an optimal control to \eqref{eq:objective-functional} subject to \eqref{eq:system} and $\phi^*$ denote the corresponding optimal state. Then, there exists a distribution $E$, zero sets of polynomials $\mc V_j\subset \CC^n, j=1,\ldots,s$, functions $Q_{j,l}, j=1,\ldots,s, l=1,\ldots,m_j$, and measures $d\mu_{j,l}$ whose support lies in $\mc V_j$, such that $u^*$ and $\phi^*$ can be represented in the following form,
\begin{equation}\label{eq:integral-representation}
   \begin{bmatrix}
      \phi^* \\
      u^*
   \end{bmatrix} = E \ast \begin{bmatrix}
      0 \\
      \phi_d
   \end{bmatrix} + \sum_{j=1}^{s}\sum_{l=1}^{m_j}\int_{\mc V_j} e^{\im k \cdot \xi} Q_{j,l}(\xi,k) d\mu_{j,l}(k).
\end{equation}
Moreover, the function $e^{\im k \cdot \xi} Q_{j,l}(\xi,k)$ satisfies the system of PDEs in \eqref{eq:optimal-control-pde} for each $k\in \mc V_j$.
\end{theorem}
\begin{proof}
According to Lemma~\ref{lemma:fundamental-solution}, there exists a distribution $E$ such that
\begin{equation}\label{eq:optimal-control-pde-fundamental}
   \begin{bmatrix}
      P(D) & -I \\
       \lambda I &P(D)^\dagger 
   \end{bmatrix}
   \left(
   E \ast
   \begin{bmatrix}
      0 \\
      \phi_d
   \end{bmatrix}
   \right)
   = \begin{bmatrix}
      0 \\
      \phi_d
   \end{bmatrix}.
\end{equation}
Combining \eqref{eq:optimal-control-pde} and \eqref{eq:optimal-control-pde-fundamental}, we have
\begin{equation}\label{eq:optimal-control-pde-rewrite}
   \begin{bmatrix}
      P(D) & -I \\
       \lambda I &P(D)^\dagger 
   \end{bmatrix}
   \left(
   \begin{bmatrix}
      \phi^* \\
      u^* 
   \end{bmatrix}
   - E \ast
   \begin{bmatrix}
      0 \\
      \phi_d   
   \end{bmatrix}
   \right)
   = 0.
\end{equation}
Applying Lemma~\ref{lemma:fundamental-principle} to \eqref{eq:optimal-control-pde-rewrite}, we obtain the integral representation in \eqref{eq:integral-representation}.
\end{proof}

\begin{example}
   The proof of Lemma~\ref{lemma:fundamental-principle} is nonconstructive in obtaining the integral representation in \eqref{eq:integral-representation}.
   Following the construction in \cite{bamieh2002distributed}, we illustrate an example of this representation for the one-dimensional heat equation \eqref{eq:heat-equation} with $\Omega_x=\RR,\Omega_t = \RR_+$, given the initial condition $\phi(x,0) = \phi_0(x)$.
   Let $\phi_d=0, \lambda=1$. 
   Following \cite[Section V-B]{bamieh2002distributed}, the optimal state and control for the heat equation can be represented as follows,
\begin{equation}\label{eq:heat-integral-representation}
   \begin{split}
         \phi^*(x,t) &= \int_{\RR} e^{\im k_x x - \omega(k_x)t} \hat{\phi}_0(k_x) \frac{dk_x}{2\pi},\\
    u^*(x,t) &= \int_{\RR} e^{\im k_x x - \omega(k_x)t} (k_x^2-\omega(k_x))\hat{\phi}_0(k_x) \frac{dk_x}{2\pi},
   \end{split}
\end{equation}
where $\omega(k_x)=\sqrt{k_x^4+1}$ and $\hat{\phi}_0(k_x)$ is the Fourier transform of the initial condition $\phi_0(x)$.
Let $k= \begin{bmatrix}
k_x & k_t
\end{bmatrix}$. Then, the expression \eqref{eq:heat-integral-representation} can be represented in the form of \eqref{eq:integral-representation} by letting $s=1,m_1=1,\mc V_1 = \{k\in\CC^2: -k_t^2 + k_x^4 +1=0\}$, $Q_{1,1}(x,t,k) = \begin{bmatrix}
1 & k_x^2 + k_t
\end{bmatrix}^\top$, $d\mu_{1,1}(k) = \frac{1}{2\pi}\hat{\phi}_0(k_x)dk$ if $k_t \in \RR_+, k_x \in \RR$ and $d\mu_{1,1}(k) = 0$ otherwise, where $dk$ is the Lebesgue measure in $\RR^2$.

For the heat equation over $\RR$, it is sufficient to represent optimal control as an integral over $\RR$ as in \eqref{eq:heat-integral-representation}. 
Note that the integral representation in Theorem~\ref{thm:integral-representation} is in general over the complex domain.
For the heat equation over a finite interval, the real domain is not enough, and optimal control can be represented as an integral over well-constructed contours in $\CC$ of the form \eqref{eq:integral-representation}, see \cite{li2024complex} for construction.
\end{example}

\section{EXTENDED FOURIER NEURAL OPERATOR}
In this section, we first present the architecture of the Fourier Neural Operator (FNO). Then, we propose an extended FNO architecture, called FNO$^{\angle \theta}$, that is inspired by the complex integral representation in \eqref{eq:integral-representation}.

\subsection{Architecture of FNO}

Let $a,b$ denote an input and output function in Banach spaces $\mc A(\Omega; \RR^{n_a}), \mc B(\Omega; \RR^{n_b})$, respectively. 
In this paper, we restrict our attention to the case where $\mc A$ and $\mc B$ are subspaces of $\text{L}^2(\Omega)\cap\text{C}^\infty(\Omega)$, respectively.
Following \cite{li2020fourier,duruisseaux2025fourier}, we consider a neural operator $\mc N(a): \mc A \to \mc B$ of the form
\begin{equation}\label{eq:neural-operator}
   \mc N(a)= \mc Q \circ \mc L_L \circ \mc L_{L-1} \circ \cdots \circ \mc L_1 \circ \mc R(a),
\end{equation}
where $L$ is a given length in $\NN$, $\mc R: \mc A(\Omega; \RR^{n_a}) \to \mc B(\Omega; \RR^{n_v})$ is a lifting operator with $n_v \geq n_a$,
$\mc Q: \mc B(\Omega; \RR^{n_v}) \to \mc B(\Omega; \RR^{n_b})$ is a projection operator,
and $\mc L_l: \mc B(\Omega; \RR^{n_v}) \to \mc B(\Omega; \RR^{n_v}), l=1,\ldots,L$ are nonlinear operator layers. 
The lifting and projection operators are typically linear transformations implemented using multilayer perceptrons \cite{duruisseaux2025fourier}, and the nonlinear operator layer is of the form,
\begin{equation}\label{eq:FNO-layer}
   (\mc L_lv_l)(\xi) = \sigma\left(W_l v_l(\xi) + \left(\mc K_l v_l\right)(\xi)\right),
\end{equation}
where $v_l\in \mc B(\Omega; \RR^{n_v})$ is the output from the previous layer, $\sigma$ is a nonlinear activation function applied component-wise, $W_l \in \RR^{n_v \times n_v}$ is a weight matrix, and $\mc K_l:\mc B(\Omega; \RR^{n_v}) \to \mc B(\Omega; \RR^{n_v})$ is a bounded linear operator.

Different types of neural operators vary by the choice of the linear operator $\mc K_l$. 
To learn solution operators of PDEs, the linear operator $\mc K_l$ in FNO is defined as a convolution operator in the Fourier space motivated by the existence of a solution in convolution form by Lemma~\ref{lemma:fundamental-solution}.
Since the convolution operation can be transformed into multiplication in the Fourier space, the operator $\mc K_l$ in FNO is expressed as the inverse Fourier transform of a multiplication between two variables $\hat{K}_l(k)$ and $\hat{v}_l(k)$,
\begin{equation}\label{eq:convolution-Kl}
(\mc K_lv_l)(\xi) = \frac{1}{(2\pi)^n}\int_{\RR^n} e^{\im k \cdot \xi} \hat{K}_l(k) \hat{v}_l(k)dk,
\end{equation}
where $\hat{K}_l(k)\in \CC^{n_v \times n_v}$ is a neural operator parameter at frequency $k$, $\hat{v}_l(k)$ is the Fourier transform of $v_l$.
\begin{remark}
   In practice, the value of input function $a$ is typically only available at a finite number of discretized spatial locations in $\Omega$.
   Also, the matrix $\hat{K}_l(k)$ is usually learned for a finite number of frequencies $k$ in the Fourier space.
Therefore, FNO is implemented in a discretized manner, where \eqref{eq:convolution-Kl} is approximated as a finite sum. 
We refer readers to \cite{li2020fourier,duruisseaux2025fourier} for more details on the discretized version of FNO and other practical implementation issues.
\end{remark}

In summary, the parameters to be learned in FNO include the weight matrices $W_l$ and the Fourier space parameters $\hat{K}_l(k)$ for all layers $l=1,\ldots,L$. 
Despite the motivation from linear PDEs, FNO has been successfully applied to learning solution operators of all different types of nonlinear PDEs, with applications in various fields \cite{li2020fourier,duruisseaux2025fourier}.
This is partially due to the expressiveness of a neural operator architecture in \eqref{eq:neural-operator}, especially when the number of layers $L$ (length) and the lifted dimension $n_v$ (width) are sufficiently large.
In the following subsection, we introduce an extended FNO architecture, called FNO$^{\angle \theta}$, which will be shown to be more expressive than FNO when $n_v$ is close to the input dimension $n_a$ in numerical experiments.

\begin{figure*}[t]
   \centering
   \begin{subfigure}[t]{0.43\textwidth}
      \centering
      \includegraphics[width=\textwidth]{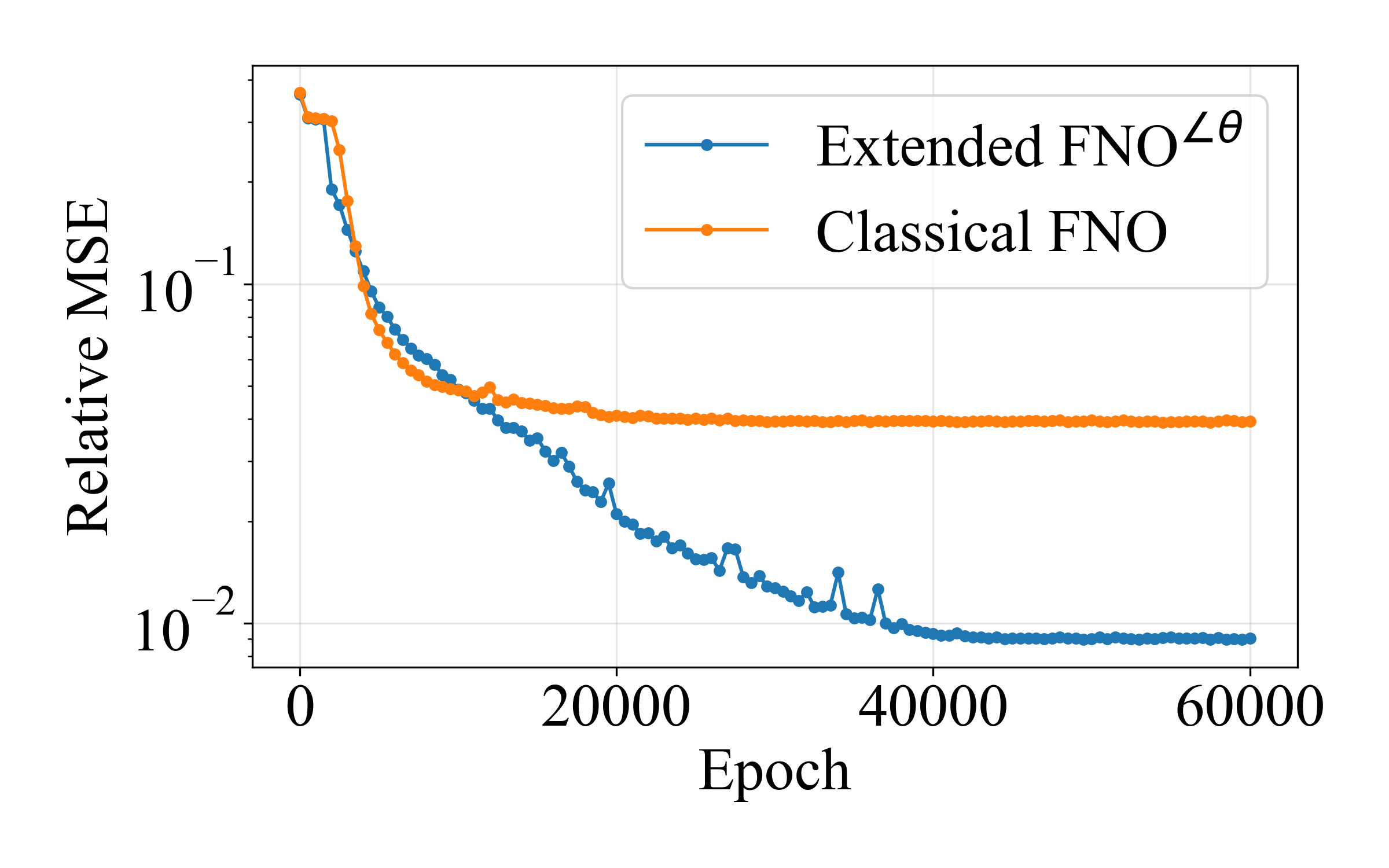}
      \caption{\sf Learning state operator}
      \label{fig:burgers_training_error_state}
   \end{subfigure}
   \hfill
   \begin{subfigure}[t]{0.43\textwidth}
      \centering
      \includegraphics[width=\textwidth]{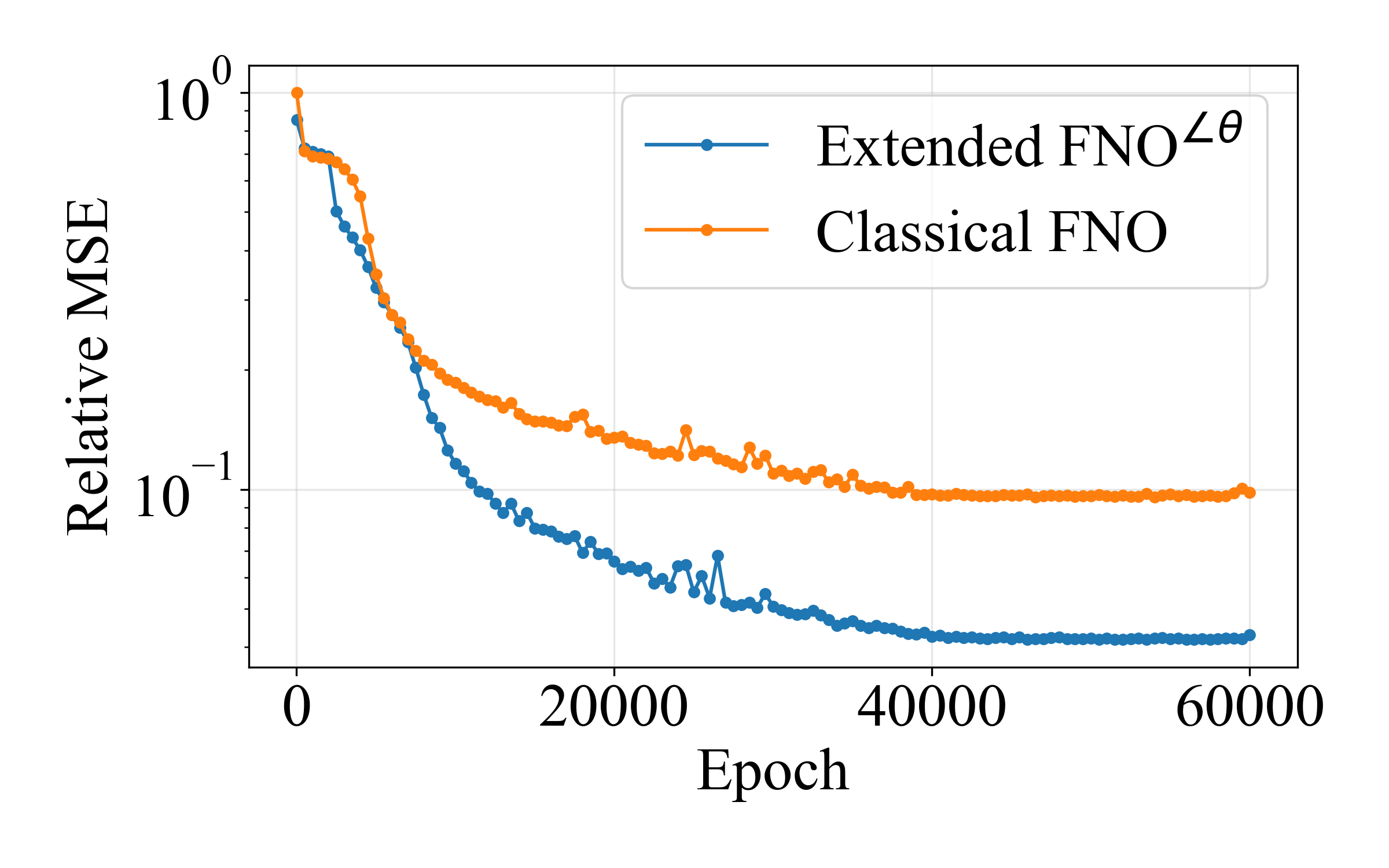}
      \caption{\sf Learning optimal control operator}
      \label{fig:burgers_training_error_control}
   \end{subfigure}
   \caption{\sf Comparison of relative mean squared errors (MSE) over all training data during the training phase of FNO and FNO$^{\angle \theta}$ for the Burgers' equation \eqref{eq:burgers-equation}.
   FNO$^{\angle \theta}$ achieves order of magnitude improvement over FNO.}
   \label{fig:burgers_training_comparison}
\end{figure*}

\subsection{The Extended FNO}

The key to the design of FNO is the convolution operator $\mc K_l$ in \eqref{eq:convolution-Kl} represented in the Fourier space, inspired by the existence of a solution in convolution form, see Lemma~\ref{lemma:fundamental-solution}.
However, not every solution to \eqref{eq:system} can be represented in convolution form, which could potentially limit the expressiveness of FNO.
In Lemma~\ref{lemma:fundamental-principle} and Theorem~\ref{thm:integral-representation}, we show that any state and optimal control of \eqref{eq:system} can be represented in a form involving a complex integral, see \eqref{eq:ehrenpreis-representation}--\eqref{eq:integral-representation}.
Comparing \eqref{eq:integral-representation} to \eqref{eq:convolution-Kl}, the integrands all contain the leading exponential term $e^{\im k \cdot \xi}$.
This motivates us to design an extended FNO architecture, where the linear operator $\mc K_l$ in \eqref{eq:FNO-layer} is given by
\begin{equation}\label{eq:convolution-Kl-extended}
(\mc K_l^{\theta}v_l)(\xi) = \frac{1}{(2\pi)^n}\int_{\RR^n} e^{\im z(k,\theta_l)\cdot \xi} \hat{K}_l(k) \hat{v}_l(k) dk,
\end{equation}
where $\theta_l = (\theta_{l,1}, \ldots, \theta_{l,n}), z_j(k,\theta_l) = k_je^{\im \theta_{l,j}},  j=1,\ldots,n$. In other words, the original frequency variable $k$ is the modulus of the new frequency variable $z$, and we introduce a new parameter $\theta_l$ to be learned as the phase of the complex variable $z$.
When the phase $\theta_l$ is zero, $\mc K_l^{\theta}$ in \eqref{eq:convolution-Kl-extended} reduces to $\mc K_l$ in \eqref{eq:convolution-Kl}.

In the following section, we present numerical experiments to show that introducing the new parameter $\theta$ greatly improves the expressiveness of the neural operator, especially when $n_v = n_a$.
Note that dimension of the new parameter $\theta$ is equal to the number of frequency variables $k$ used to evaluate the integral in \eqref{eq:convolution-Kl-extended}, and is independent of the lifted dimension $n_v$ unlike the other parameters $\hat{K}_l(k)$ and $\hat{v}_l(k)$.

\section{Numerical Experiments}

In this section, we present numerical results to compare the performance of FNO and FNO$^{\angle \theta}$ in learning state and optimal control for the Burgers' equation.
Consider the domain $\Omega = \Omega_x \times \Omega_t$ where $\Omega_x = [0,1]$ and $\Omega_t=[0,0.5]$. The Burgers' equation can be written in the form,
\begin{equation}\label{eq:burgers-equation}
   \frac{\partial \phi(x,t)}{\partial t} + \phi(x,t) \frac{\partial \phi(x,t)}{\partial x} = \nu \frac{\partial^2 \phi(x,t)}{\partial x^2} + u(x,t),
\end{equation}
where $\nu$ is the viscosity coefficient.
We choose a small value of $\nu=0.05$ that generates nonsmooth-like shock wave behavior to increase complexity for training neural operators.
We impose the following initial and boundary conditions for \eqref{eq:burgers-equation} to ensure the uniqueness of the solution,
\begin{equation}\label{eq:burgers-initial-bondition}
   \phi(x,0) = \phi_0(x), \quad x\in [0,1].
\end{equation}
\begin{equation}\label{eq:burgers-boundary-condition}
   \phi(0,t) = g(t), \phi(1,t) = h(t), \quad t\in [0,0.5],
\end{equation}
where $\phi_0, g, h$ are given smooth functions.

The Burgers' equation \eqref{eq:burgers-equation} is a nonlinear PDE that has been widely used as a benchmark problem for testing the performance of neural operators \cite{li2020fourier,takamoto2022pdebench}.
Note that \eqref{eq:burgers-equation} cannot be written in the form of \eqref{eq:system} due to the nonlinearity. 
Therefore, the integral representation in Theorem~\ref{thm:integral-representation} does not directly apply to \eqref{eq:burgers-equation}.
However, the expressiviness of the neural operator architecture in \eqref{eq:neural-operator} allows us to learn approximated representations for solution and optimal control operators of \eqref{eq:burgers-equation} using FNO and FNO$^{\angle \theta}$. 

\subsection{Learning state operator}
\label{sec:numerical-state-operator}
\begin{figure}[b]
   \centering
   \begin{subfigure}[t]{0.23\textwidth}
      \centering
      \includegraphics[width=\textwidth]{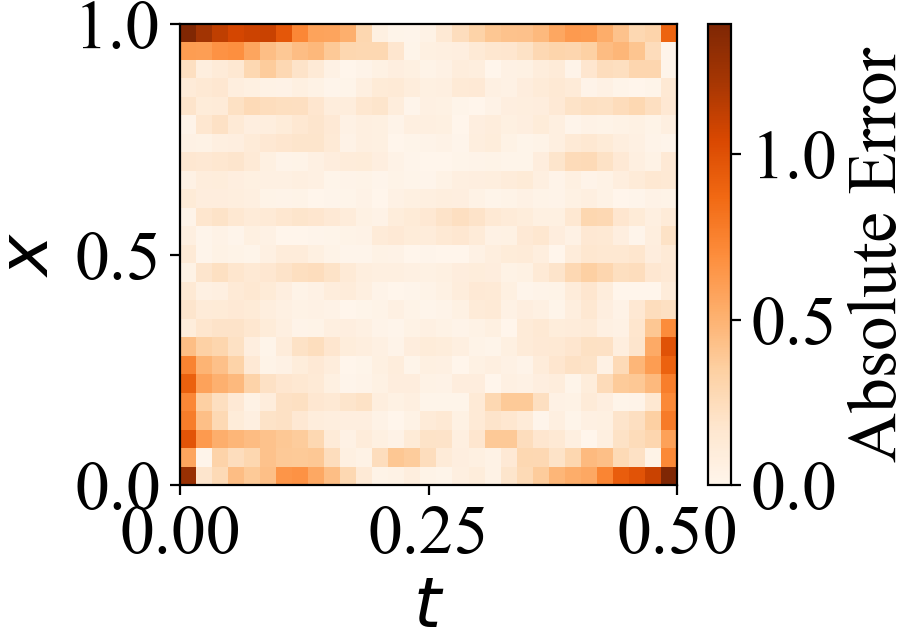}
      \caption{\sf FNO}
      \label{fig:burgers_state_comparison_FNO}
   \end{subfigure}
   \hfill
   \begin{subfigure}[t]{0.23\textwidth}
      \centering
      \includegraphics[width=\textwidth]{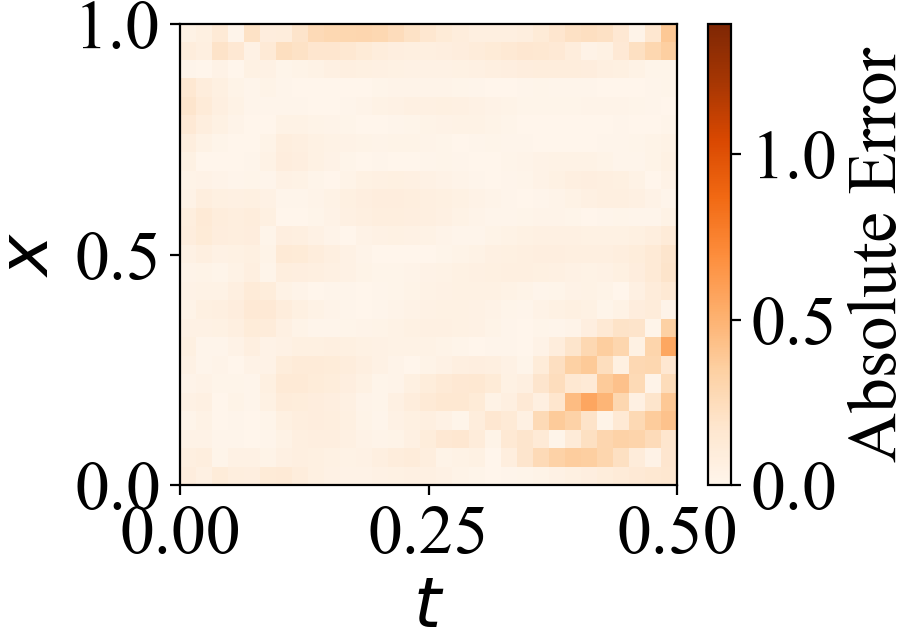}
      \caption{\sf FNO$^{\angle \theta}$}
      \label{fig:burgers_state_comparison_extended_FNO}
   \end{subfigure}
   \caption{\sf Learning state operator: comparison of the absolute error between the neural operator output and the training data for a particular boundary condition. The error for FNO is concentrated on the boundaries, whereas FNO$^{\angle \theta}$ generates more accurate boundary values.}
   \label{fig:burgers_state_comparison}
\end{figure}

First, we train neural operators of the form \eqref{eq:neural-operator} to learn the state operator of the Burgers' equation, which takes the boundary conditions $g,h$ as input $a$ and gives the state $\phi$ as output $b$. 
We follow a supervised learning approach as in \cite{li2020fourier}, where we minimize the mean squared error (MSE) between the predicted state and the true state over a training dataset.
To generate the training dataset, we discretize the domain $\Omega$ with the grid size $n_x = 24$ and $n_t=30$ for space and time domains $\Omega_x$ and $\Omega_t$, respectively.
We fix the initial condition as $\phi_0(x) = \sin(\pi x)$ and randomly sample the boundary conditions $g(t), h(t)$ from a Gaussian random field \cite{Rasmussen2004} with the length scale $\ell=0.15$ and standard deviation $\sigma=1$. 
Then, we solve the Burgers' equation \eqref{eq:burgers-equation} with the sampled boundary conditions using the finite difference method following \cite{li2020fourier} to obtain the corresponding state $\phi(x,t)$ in $\Omega$, which serves as the ground truth for training.
The training dataset consists of 50 state trajectories under different sampled boundary conditions.

For both FNO and FNO$^{\angle \theta}$, we use the Gaussian Error Linear Unit (GELU) activation function for $\sigma$ and batch normalization following \cite{duruisseaux2025fourier}. The number of Fourier layers is set to $L=4$, according to the recommendation in \cite[Section 4.3]{duruisseaux2025fourier}.
The lifted dimension is set to $n_v=n_a$. 
The lifted dimension $n_v$ determines the expressive capacity of the neural operator. Increasing $n_v$ allows FNO to represent more complex relations in the data, but comes at the cost of increased computational complexity and memory usage \cite{duruisseaux2025fourier}.
In our experiments, we choose $n_v=n_a$ to evaluate how much performance improvement FNO$^{\angle \theta}$ can achieve when FNO is not performing well.

To implement the operators \eqref{eq:convolution-Kl} and \eqref{eq:convolution-Kl-extended}, we discretize the Fourier space and learn the parameters $\hat{K}_l(k)$ and $\theta_l(k)$ for a finite number of frequencies $k_m, m= -M, -M +1, \ldots, M$.
Then, the integrals in \eqref{eq:convolution-Kl} and \eqref{eq:convolution-Kl-extended} are approximated as sums over the discretized frequencies.
To be consistent with \cite{li2020fourier}, we choose the frequency $k_m \in \mathbb{Z}$, which leads to the form of discrete Fourier series.
In practice, a large value of $M$ typically introduces more errors during the training \cite[Section 4.2]{duruisseaux2025fourier}, and we set $M=4$ in our experiments, i.e., $k_m = m, m=-4, -3,\ldots, 4$.

\begin{figure}[b]
   \centering
   \begin{subfigure}[t]{0.23\textwidth}
      \centering
      \includegraphics[width=\textwidth]{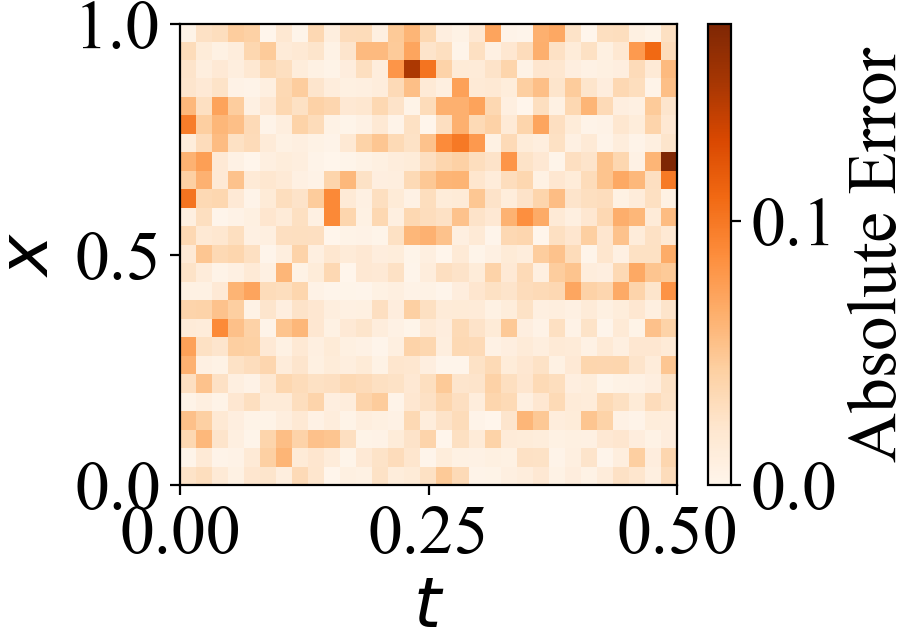}
      \caption{\sf FNO}
      \label{fig:burgers_control_comparison_FNO}
   \end{subfigure}
   \hfill
   \begin{subfigure}[t]{0.23\textwidth}
      \centering
      \includegraphics[width=\textwidth]{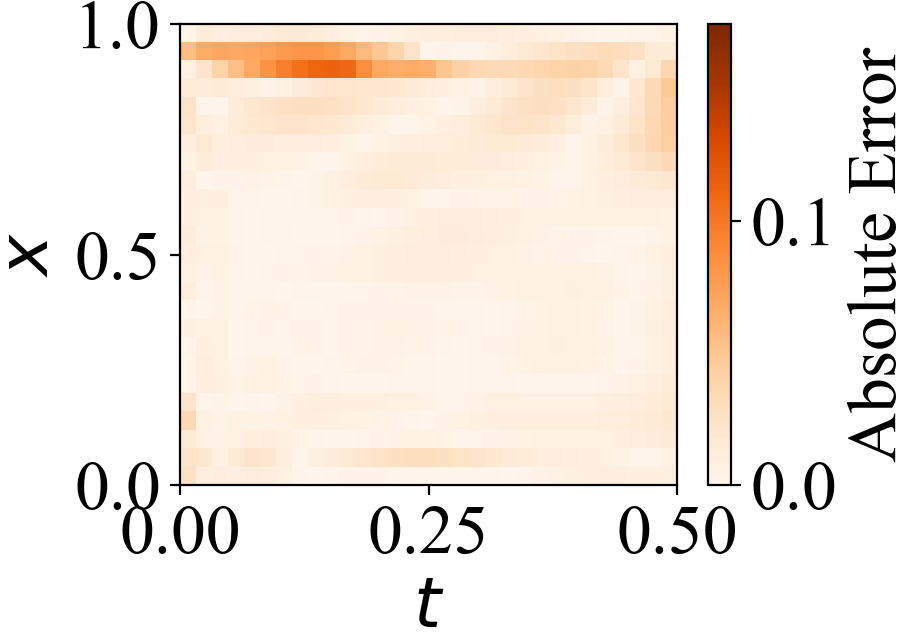}
      \caption{\sf FNO$^{\angle \theta}$}
      \label{fig:burgers_control_comparison_extended_FNO}
   \end{subfigure}
   \caption{\sf Learning optimal control operator: comparison of the absolute error between the neural operator output and the training data for a particular boundary condition. 
   The error for FNO is scattered across the domain, whereas the error for FNO$^{\angle \theta}$ is mostly smaller.}
   \label{fig:burgers_control_comparison}
\end{figure}

Figure~\ref{fig:burgers_training_error_state} compares the relative mean squared errors (MSE) during the training phase of FNO and FNO$^{\angle \theta}$ for learning state operator.
The values of relative MSE after training are 0.039 and 0.0089 for FNO and FNO$^{\angle \theta}$, respectively, showing an order of magnitude improvement with the phase parameter.
To gain further insights on the improvement, Figure~\ref{fig:burgers_state_comparison} shows the absolute error between the outputs from trained neural operators and the training data for a particular boundary condition.
It can be seen from Figure~\ref{fig:burgers_state_comparison_FNO} that most of the errors for FNO are concentrated on the four boundaries $x=0, x=1, t=0, t=0.5$.
This is because discretizing the Fourier transform \eqref{eq:convolution-Kl} with integer frequency $k_m$ results in a Fourier series, and the Fourier series generates periodic boundary values that are inconsistent with the input boundary condition.
It is possible for FNO to approximate non-periodic boundary conditions using the weight matrix $W_l$ in the operator layer \eqref{eq:FNO-layer} \cite{li2020fourier}; however, a large value of the lifted dimension $n_v$ is needed to increase the expressivity of the neural operator.
In contrast, Figure~\ref{fig:burgers_state_comparison_extended_FNO} shows that FNO$^{\angle \theta}$ generates boundary values that are more consistent with the input boundary conditions than FNO.
This is because extending the frequency $k$ in \eqref{eq:convolution-Kl} to complex variable $z$ in \eqref{eq:convolution-Kl-extended} allows representing a richer class of PDE solutions in potentially bounded domains.

\subsection{Learning optimal control operator}

In addition to learning state operator, we also compare the performance of FNO and FNO$^{\angle \theta}$ for learning optimal control operator of the burgers' equation \eqref{eq:burgers-equation}.
Following the same setting in Section~\ref{sec:numerical-state-operator}, we generate the training dataset for learning operators that map boundary conditions to optimal control.
The boundary conditions are sampled from the same Gaussian random field in Section~\ref{sec:numerical-state-operator}, and the corresponding optimal control is computed using the adjoint method in \cite{fikl2016comprehensive}.

Figure~\ref{fig:burgers_training_error_control} compares the relative MSE during the training phase of FNO and FNO$^{\angle \theta}$ for learning optimal control operator.
The values of relative MSE after training are 0.098 and 0.042 for FNO and FNO$^{\angle\theta}$, respectively.
Figure~\ref{fig:burgers_control_comparison} shows the absolute error between the trained operator output and the training data for a particular boundary condition.
As shown in Figure~\ref{fig:burgers_control_comparison_FNO}, the error is scattered across the space-time domain $\Omega$, meaning that FNO is not able to capture the complex relation between the input boundary condition and the corresponding optimal control.
For the FNO$^{\angle \theta}$, the absolute error is largely reduced for most part of the domain except for the top area close to the boundary $x=1$.
Further investigation is needed to analyze the cause of this error.



\section{CONCLUSIONS AND FUTURE WORKS}

\subsection{Conclusions}
We analyze the representations for state and optimal control of systems governed by linear partial differential equations with constant coefficients.
For this class of distributed parameter systems, we show that state and optimal control can be represented in terms of integrals over the complex domain involving the same exponential term as in the inverse Fourier transform. Therefore, such representations can be regarded as extending the Fourier frequency space from real to complex domains.
We use this insight to redesign the convolution operator in FNO, originally represented using the inverse Fourier transform. 
We introduce a new complex variable in the inverse transform, where the modulus is the original real frequency, and the phase is a new neural operator parameter to be learned. 
Numerical experiments for learning state and optimal control operators of the nonlinear burgers' equation illustrate that order of magnitude improvement for the training error can be achieved by introducing the phase parameter.
Higher consistency between output boundary values and the given non-periodic boundary conditions is also observed comparing FNO$^{\angle \theta}$ to FNO.

\subsection{Future Works}
So far, we have only evaluated the performance of FNO$^{\angle \theta}$ for a single one-dimensional PDE. We plan to test the performance on extensive benchmarks (see for example \cite{takamoto2022pdebench}) for various types of PDEs. Beyond supervised settings, it would also be interesting to investigate unsupervised training schemes since optimal control data are usually not available.



\bibliography{reference}
\bibliographystyle{IEEEtran}

\end{document}